\documentclass{article}

\PassOptionsToPackage{numbers,sort&compress}{natbib}
\usepackage[preprint]{neurips_2026}
\usepackage[utf8]{inputenc} % allow utf-8 input
\usepackage[T1]{fontenc}    % use 8-bit T1 fonts
\usepackage{hyperref}       % hyperlinks
\usepackage{url}            % simple URL typesetting
\usepackage{booktabs}       % professional-quality tables
\usepackage{amsfonts}       % blackboard math symbols
\usepackage{amsmath}
\usepackage{nicefrac}       % compact symbols for 1/2, etc.
\usepackage{microtype}      % microtypography
\usepackage{graphicx}
\usepackage{amsthm}
\usepackage{amssymb}
\usepackage{tabularx}
\usepackage{xcolor}
\title{B3O: Scalable Boltzmann Batch Bayesian Optimization}

\author{%
  Maximilian Bloor\\
  Imperial College London\\
  London, UK\\
  \texttt{max.bloor22@imperial.ac.uk}
  \And
  Liyuan Xu\\
  Secondmind\\
  Cambridge, UK\\
  \texttt{liyuan.xu@secondmind.ai}\\
  \AND
  Hrvoje Stojic\\
  Secondmind\\
  Cambridge, UK\\
  \texttt{hrvoje.stojic@secondmind.ai}\\
  \And
  Victor Picheny\\
  Secondmind\\
  Cambridge, UK\\
  \texttt{victor@secondmind.ai}\\
}

\begin{document}

\maketitle

\begin{abstract}
Modern engineering workflows increasingly rely on massive parallel simulation, driving the need for scalable, large-batch Bayesian Optimization (BO). Existing batch BO methods, however, incur large computational cost or rely on approximations that erode batch diversity. We propose B3O (Boltzmann Batch Bayesian Optimization), a framework that reframes batch generation as a pure sampling problem: drawing samples directly from the Boltzmann distribution defined by the acquisition function avoids the bottlenecks of existing large-batch methods. Theoretically, we prove that queries sampled from this distribution incur only negligible additional regret. Empirically, B3O outperforms existing batch BO methods on standard synthetic benchmarks and adapts robustly across complex applied tasks, including multi-objective electrode design and mixed-variable race car configuration.
\end{abstract}

\newtheorem{assumption}{Assumption}
\newtheorem{theorem}{Theorem}

\section{Introduction}

From simulation-driven engineering design to high-throughput molecular screening, expensive black-box objectives are increasingly evaluated in parallel,   
making batch Bayesian Optimization (BO), where $B$ candidates are queried concurrently per iteration, the regime of practical interest. Fully exploiting this parallelism while preserving sample efficiency is the central challenge as $B$ scales, and existing strategies struggle in complementary ways.

Conventional multi-point acquisitions such as $q$-Expected Improvement \citep{ginsbourger2010kriging} target the small-batch regime: they evaluate high-dimensional expectations over the joint posterior of the $B$ candidates, with each Monte Carlo evaluation requiring an $\mathcal{O}(B^3)$ Cholesky factorization and joint optimization over the $B\cdot d$-dimensional product space, which becomes intractable as $B$ grows. Sequential greedy approximations such as $q$-UCB \citep{wilson2017reparameterization} or local penalization \citep{gonzalez2016batch} sidestep this joint optimization at $\mathcal{O}(B)$ inner-optimization cost, but rely on heuristics that struggle to capture joint uncertainties and produce clustered, suboptimal batches.

Scalable Thompson Sampling (TS) \citep{thompson1933likelihood, vakili2021scalable, wilson2020efficiently, moss2023inducing} sidesteps the joint posterior by optimizing one posterior trajectory per batch element. Practical implementations rely on spectral approximations such as Random Fourier Features \citep{hernandezlobato2014predictiveentropysearch} that suffer from variance starvation \citep{mutny2019efficient}, collapsing the optimized trajectories onto a narrow set of modes, and the decomposition is only compatible with stationary kernels, limiting its applicability. Diversity correctives based on Determinantal Point Processes \citep{kulesza2012DPP, kathuria2016batched, nava2022diversifiedsamplingbatchedbayesian} reintroduce a cost cubic in $B$ and cannot recover modes never visited. These limitations motivate a truly scalable batch BO algorithm, robust across surrogate and acquisition choices, that preserves multimodal exploration of the acquisition surface.

Boltzmann exploration offers a principled alternative: queries are drawn from a density built on a utility function, with a temperature parameter trading off greediness and exploration. The mechanism has strong theoretical backing in discrete multi-armed bandits \citep{cesa2017boltzmann}, but extending it to continuous BO requires sampling from intractable, multimodal acquisition densities.

\citet{garcia2019fully, garciabarcos2025advanced} explored acquisition-density sampling for BO using Metropolis--Hastings and gradient-based MCMC in distributed and small-batch settings. Stein Boltzmann Sampling \citep{serre2025stein} applies Stein Variational Gradient Descent to a Boltzmann density defined over the objective itself, but requires objective gradients and is therefore inapplicable to BO's derivative-free black-box setting.

We introduce B3O, a scalable batch BO framework that recasts batch generation as sampling. Drawing $B$ i.i.d.\ samples from a Boltzmann density on the marginal acquisition makes B3O scale linearly in $B$ while preserving multimodal exploration of the acquisition landscape, sidestepping the scalability--diversity trade-off of joint-posterior and trajectory-based approaches (Table~\ref{tab:comparison_table}). It is also flexible: continuous, constrained multi-objective, and mixed continuous--discrete search spaces are all handled by the same algorithm, swapping the sampler or the acquisition as needed. Our contributions:
\begin{enumerate}
    \item \textbf{Method:} We frame batch generation as sampling from a Boltzmann density built solely from marginal predictive distributions, achieving linear batch-size scaling. We show that both a fixed inverse temperature and a time-varying schedule yield reliable optimization.
    \item \textbf{Theory:} We prove a finite-time cumulative regret bound for sampling from the Boltzmann distribution defined by an UCB acquisition that recovers standard GP-UCB rates up to a negligible additive term.
    \item \textbf{Empirical generality:} We demonstrate the wide applicability of B3O, which is sampler-, surrogate-, and acquisition-agnostic, through empirical experiments across continuous synthetic benchmarks, a constrained multi-objective electrode design problem (by changing the acquisition to expected hypervolume improvement), and a mixed continuous--discrete configuration problem (by changing the sampler to a mixed-space variant of Metropolis--Hastings).
\end{enumerate}

\begin{table}[htbp]
\centering
\caption{Qualitative comparison of B3O with standard Batch BO baselines. $^{\dagger}$Multi-point batch acquisitions such as $q$-EI provide diversity implicitly via the joint posterior, but at a $B \cdot d$-dimensional inner optimization cost.}
\label{tab:comparison_table}
% This line scales the table to the text width
\resizebox{\textwidth}{!}{%
\begin{tabular}{@{} c c c c c c @{}}
\toprule
Method & \begin{tabular}{@{}c@{}}Black-box \\ (no gradients)\end{tabular} & \begin{tabular}{@{}c@{}}Large \\ Batches\end{tabular} & \begin{tabular}{@{}c@{}}Diversity \\ Control\end{tabular} & \begin{tabular}{@{}c@{}}Various BO Settings \\ (e.g. multi-objective)\end{tabular} & \begin{tabular}{@{}c@{}}Model \\ Agnostic\end{tabular} \\ \midrule
Multi-Point Batch Alg. (e.g. $q$-EI) \citep{ginsbourger2010kriging} & \checkmark & \texttimes & (\checkmark)$^{\dagger}$ & \checkmark & \checkmark \\
TS \citep{vakili2021scalable} & \checkmark & \checkmark & \texttimes & \checkmark & \texttimes \\
SBS \citep{serre2025stein} & \texttimes & \checkmark & \checkmark & \texttimes & \texttimes \\
\textbf{B3O (Proposed)} & \checkmark & \checkmark & \checkmark & \checkmark & \checkmark \\ \bottomrule
\end{tabular}%
}
\end{table}

\begin{figure}[t]
    \centering
    \includegraphics[width=0.7\linewidth]{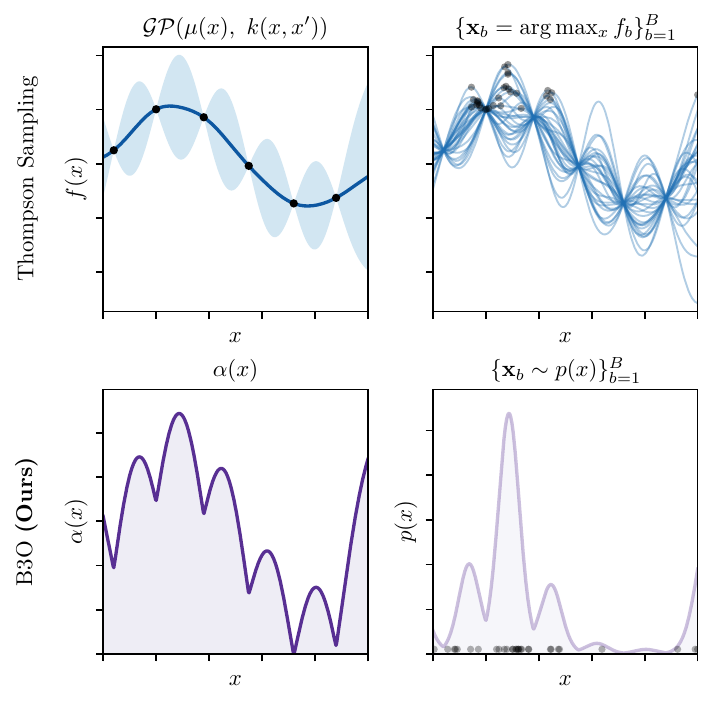}
    \caption{Comparison of TS and B3O for 1D batch generation, sharing the same underlying GP posterior. \textbf{Top row (TS):} (left) the GP posterior; (right) query points obtained by optimizing each sampled trajectory. \textbf{Bottom row (B3O):} (left) the marginal Upper Confidence Bound acquisition; (right) a diverse batch of query points drawn from the corresponding Boltzmann density.}
    \label{fig:b3o}
\end{figure}

\section{Problem Formulation}

We consider sequential optimization of an expensive function $f$ over a compact set $\mathcal{X}\subset \mathbb{R}^d$. At each timestep $t = 1, \ldots, T$, the algorithm selects a batch $\{\mathbf{x}_{t,b}\}_{b\in [B]}$ that is evaluated in parallel, returning noisy rewards $\{y_{t,b} = f(\mathbf{x}_{t,b}) + \epsilon_{t,b}\}_{b\in[B]}$ with observation noise $\epsilon_{t,b}$.

Letting $\mathbf{x}^* \in \arg \max_{\mathbf{x}}f(\mathbf{x})$, the regret over horizon $T$ is the cumulative gap to $f(\mathbf{x}^*)$,
\begin{equation}
    R(T,B;f) = \mathbb{E}\left[\sum_{t=1}^T\sum_{b=1}^B f(\mathbf{x}^*)-f(\mathbf{x}_{t,b})\right],
\end{equation}
where the expectation is over the observation noise and the algorithm's internal sampling randomness. We assume the regularity conditions in Assumption~\ref{ass:regularity}.
\begin{assumption}\label{ass:regularity}
    Given a reproducing kernel Hilbert space (RKHS), the norm of the objective function is bounded: $\|f\|_{H_k}\leq M_{\mathcal{H}}$, for some $M_{\mathcal{H}}>0$, and $k(\mathbf{x},\mathbf{x}')\leq 1$ for all $\mathbf{x}, \mathbf{x}'\in \mathcal{X}$. The observation noise $\epsilon_{t,b}$ is assumed to be zero-mean and $\sigma$-sub-Gaussian, independent across $(t,b)$.
\end{assumption}

\section{Bayesian Optimization Background}
BO is a sequential decision-making strategy for an expensive black-box $f$, combining a probabilistic surrogate (e.g.\ Gaussian Processes~\citep{rasmussen2010gaussian}, Bayesian neural networks~\citep{li2023study}, deep ensembles~\citep{lakshminarayanan2017simple}) with an acquisition function that, when maximized, decides where to sample next.

\subsection{Probabilistic Models}
A surrogate provides a posterior $p(f \mid \mathcal{D}_t)$ over the function space given observed data $\mathcal{D}_t = \{(\mathbf{x}_i, y_i)\}_{i=1}^{N_t}$. For most acquisition functions, we only require the marginal predictives $p(y \mid \mathbf{x}, \mathcal{D}_t)$, characterized by a mean $\mu_t(\mathbf{x})$ and variance $v_t(\mathbf{x})$. The canonical choice is the Gaussian Process (GP) (Figure~\ref{fig:b3o}, top-left), but any model providing marginal predictives (e.g. deep kernel learning, Bayesian neural networks, or ensembles) can be substituted.

For a set of noisy observations $\mathbf{y}_t$, the GP posterior at a point $\mathbf{x}$ is Gaussian, $\mathcal{N}(\mu_t(\mathbf{x}), v_t(\mathbf{x}))$, with predictive mean and variance
\begin{equation}
    \mu_t(\mathbf{x}) = \mathbf{k}_{t}(\mathbf{x})^\intercal (\mathbf{K}_{t} + \sigma_n^2\mathbf{I})^{-1} \mathbf{y}_{t}, \qquad v_t(\mathbf{x}) = k(\mathbf{x}, \mathbf{x}) - \mathbf{k}_{t}(\mathbf{x})^\intercal (\mathbf{K}_{t} + \sigma_n^2\mathbf{I})^{-1} \mathbf{k}_{t}(\mathbf{x}).
\end{equation}

Exact GPs are the standard for uncertainty quantification, but their $\mathcal{O}(N_t^3)$ cost from covariance-matrix inversion becomes a bottleneck as data grows. We therefore use Sparse Variational Gaussian Processes (SVGPs)~\citep{titsias2009variational}, which use $m$ inducing variables $\mathbf{u}$ at locations $\mathbf{Z}_t$ to reduce complexity to $\mathcal{O}(N_t m^2)$, scaling to the large budgets targeted by B3O without sacrificing principled uncertainty, with inducing-point allocation strategies tailored to high-throughput BO~\citep{moss2023inducing}.

\subsection{Acquisition Functions}
The acquisition function $\alpha_t(\mathbf{x})$ maps the posterior to a utility value (Figure~\ref{fig:b3o}, bottom-left), and the next query is $\mathbf{x}_{t+1} \in \arg \max_{\mathbf{x} \in \mathcal{X}} \alpha_t(\mathbf{x})$. Many standard acquisitions are point-wise operators relying solely on marginal predictives. A canonical example is Log Expected Improvement (LogEI) \citep{ament2023unexpected}, a numerically stable reformulation of Expected Improvement \citep{jones1998efficient} that prioritizes regions expected to exceed an incumbent $\eta_t$,
\begin{equation}
    \alpha_{\text{LogEI}}(\mathbf{x}) = \log \mathbb{E}_{f\sim p(f\mid\mathcal{D}_t)}\!\left[\max(0, f(\mathbf{x}) - \eta_t)\right].
\end{equation}
Under noisy observations, the latent best $f(\mathbf{x}^+)$ is unobservable, so $\eta_t$ is a deterministic scalar (e.g.,\ the best posterior mean $\max_{\mathbf{x}\in\mathcal{D}_t}\mu_t(\mathbf{x})$ or the best noisy observation $\max_i y_i$) preserving the marginal-predictive form of the acquisition.
Another canonical choice is the Upper Confidence Bound (UCB) \citep{srinivas2010gaussian}, which explicitly trades off exploration and exploitation through a parameter $\beta_t$,
\begin{equation}\label{eq:ucb}
    \alpha_{\text{UCB}}(\mathbf{x}) = \mu_{t}(\mathbf{x}) + \beta_{t}^{1/2}\sigma_{t}(\mathbf{x}),
\end{equation}
where $\sigma_t(\mathbf{x}) = \sqrt{v_t(\mathbf{x})}$. Selecting the next query requires inner-maximizing $\alpha_t$. While $\alpha_t$ is much cheaper than $f$, its landscape is typically non-convex, multimodal, and contains large flat regions, making global maximization non-trivial.

\subsection{Batch Bayesian Optimization}\label{sec:batch_bo}
At each iteration, $B$ candidates $\{\mathbf{x}_{t,b}\}_{b=1}^{B}$ are selected for parallel evaluation. Existing strategies share an outer structure---one or more maximizations of an acquisition or sampled trajectory---differing in the mechanism for inducing diversity and the inner-optimization cost. We highlight three regimes.

\paragraph{Joint-posterior multi-point acquisitions.}
Methods such as $q$-EI \citep{ginsbourger2010kriging} and $q$-UCB \citep{JMLR:v15:desautels14a}, where $q\equiv B$, build a multi-point acquisition $\alpha^B_t : \mathcal{X}^B \to \mathbb{R}$ from the joint posterior and solve
\begin{equation}\label{eq:qei}
    \{\mathbf{x}_{t,b}\}_{b=1}^{B} \in \arg\max_{(\mathbf{x}_1,\ldots,\mathbf{x}_B)\, \in\, \mathcal{X}^B}\,\alpha^B_t(\mathbf{x}_1,\ldots,\mathbf{x}_B).
\end{equation}
Diversity arises implicitly: correlated predictive variances penalize collocated batches. Two practical costs follow. First, evaluating $\alpha^B_t$ requires Monte Carlo estimation of a $B$-dimensional joint expectation that involves a Cholesky factorization of the $B\times B$ joint covariance, yielding $\mathcal{O}(B^3)$ cost per evaluation. Second, \eqref{eq:qei} maximizes this surface over the $B\cdot d$-dimensional product space $\mathcal{X}^B$, which becomes intractable at the batch sizes we target. Greedy heuristics \citep{ginsbourger2010kriging} replace this with a sequence of single-point maximizations, restoring tractability but reintroducing clustering when the posterior is poorly informed.

\paragraph{Repulsion-based heuristics.}
Local penalization \citep{gonzalez2016batch} keeps the inner optimization single-point and enforces diversity explicitly via a repulsion factor
\begin{equation}
    \mathbf{x}_{t,b} \in \arg\max_{\mathbf{x}\in\mathcal{X}}\;\alpha_t(\mathbf{x})\,\prod_{b'<b}\phi(\mathbf{x},\, \mathbf{x}_{t,b'}),
\end{equation}
where $\phi$ depends on a Lipschitz estimate of $f$. This requires $B$ sequential single-point acquisitions---linear in $B$ but unparallelisable across batch elements---and relies on a Lipschitz estimate that is hard to obtain on multimodal landscapes.

\paragraph{Trajectory sampling methods.}
Scalable Thompson Sampling \citep{vakili2021scalable, wilson2020efficiently, moss2023inducing} draws $B$ independent trajectories $\tilde f^{(b)}_t \sim p(f \mid \mathcal{D}_t)$, $b = 1,\dots,B$, from the surrogate posterior and solves $B$ independent single-point optimizations,
\begin{equation}\label{eq:ts}
    \mathbf{x}_{t,b} \in \arg\max_{\mathbf{x}\in\mathcal{X}}\, \tilde f^{(b)}_t(\mathbf{x}),
\end{equation}
which is parallelisable across $B$ (Figure~\ref{fig:b3o}, top row), but each batch element still requires its own non-trivial optimisation of a typically multimodal trajectory. Diversity is then sample-based: it relies on the trajectories being sufficiently distinct, which in turn requires faithful global posterior samples. For a GP, exact posterior draws are only available at a finite, pre-specified set of inputs, reducing \eqref{eq:ts} to discrete maximisation over that grid and so failing to scale to fine continuous discretisations. To recover continuous-input trajectories, practical implementations approximate each $\tilde f^{(b)}_t$ with a finite-dimensional surrogate---most commonly Random Fourier Features \citep{hernandezlobato2014predictiveentropysearch}, or the pathwise updates of \citet{wilson2020efficiently} that combine prior samples with a Matheron-style data update. Both pipelines rely on a finite linear decomposition of the kernel (e.g.,\ a Bochner-type spectral basis), which under-represents posterior variance away from the data---the variance-starvation phenomenon analyzed by \citet{mutny2019efficient}---so as the budget grows the trajectories tend to share modes even though their inner optimisations are formally independent. Because this decomposition is a property of the GP kernel itself, the same scalable-TS pipeline does not transfer to surrogates without one (e.g.,\ Bayesian neural networks or deep ensembles), restricting scalable TS to GP surrogates with stationary kernels.

B3O replaces the inner optimization altogether with a sampling step on the marginal acquisition (Section~\ref{sec:boltzmann}): multimodal coverage arises directly from the Boltzmann density, without a joint posterior, repulsion penalty, or trajectory sample. Intra-batch diversity is then a probabilistic property of the sampling distribution---ensured by sufficient temperature---rather than a structural guarantee against redundant draws.

%%%%%%%%%%%%%%%%%%%%%%%%%%%%%%%
\section{Boltzmann Distribution}\label{sec:boltzmann}

\paragraph{Boltzmann sampling on a generic utility.} Given a utility function $u : \mathcal{X} \to \mathbb{R}$ and inverse temperature $\lambda > 0$, the Boltzmann (or Gibbs) distribution is
\begin{equation}\label{eq:boltzmann_generic}
    p(\mathbf{x}) \propto \exp(\lambda\, u(\mathbf{x})).
\end{equation}
This construction is foundational in reinforcement learning \citep{sutton1998introduction} and multi-armed bandits \citep{cesa2017boltzmann}. The inverse temperature $\lambda$ interpolates between two extremes: as $\lambda \to 0$ the density approaches uniform and samples ignore $u$; as $\lambda \to \infty$ it concentrates at the global maximizer of $u$. Changing this parameter $\lambda$ therefore tunes the trade-off between diversity and exploitation; see Figure~\ref{fig:lambda_effect}.

\begin{figure}
    \centering
    \includegraphics[width=0.9\linewidth]{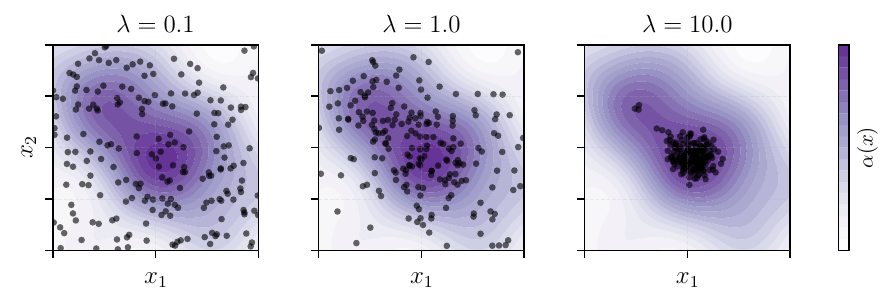}
    \caption{Impact of inverse temperature $\lambda$ on Boltzmann batch sampling across a 2D Upper Confidence Bound acquisition $\alpha(\mathbf{x})$. Black markers denote sampled points. A low $\lambda$ (0.1) promotes diversity, a moderate $\lambda$ (1.0) balances exploration with exploitation, and a high $\lambda$ (10.0) concentrates samples at the maxima.}
    \label{fig:lambda_effect}
\end{figure}

\paragraph{Specialization to batch BO.} B3O instantiates \eqref{eq:boltzmann_generic} with $u = \alpha_t$, the chosen point-wise acquisition (EI, UCB, expected hypervolume improvement, \ldots), and draws $B$ i.i.d.\ samples to form the batch $\mathcal{B}_t = \{\mathbf{x}_{t,b}\}_{b=1}^{B}$ (Figure~\ref{fig:b3o}, bottom row),
\begin{equation}\label{eq:boltzmann_pdf}
    p_t(\mathbf{x}) = \frac{\exp\left( \lambda_t\, \alpha_t(\mathbf{x}) \right)}{Z_t(\lambda_t)}, \qquad Z_t(\lambda_t) = \int_{\mathcal{X}} \exp(\lambda_t\, \alpha_t(\mathbf{x}))\,\mathrm{d}\mathbf{x}.
\end{equation}

\paragraph{Inverse temperature.} We index $\lambda$ with $t$ to allow time-varying schedules: a low initial $\lambda$ encourages exploration when the surrogate is least informed, and growing $\lambda_t$ concentrates sampling on high-acquisition regions as the surrogate converges (Figure~\ref{fig:lambda_effect}). Our analysis (Section~\ref{sec:regret_analysis}) shows that $\lambda_t$ must grow at rate at least $\sqrt{t}\log t$ for B3O-UCB to inherit the GP-UCB rate, identifying increasing temperature as theoretically necessary. In practice, the schedule is not strictly required: a well-chosen constant $\lambda$ is competitive with, and sometimes outperforms, the schedule (Section~\ref{sec:experiments}), so the algorithm is robust to the scheduling of $\lambda_t$.

\paragraph{Properties of the sampling formulation.} Replacing the inner optimization steps of the batch strategies above with a single sampling step has several practical consequences. First, once the sampler is constructed at iteration $t$, generating the batch requires $B$ i.i.d.\ draws, scaling linearly in $B$ (matching scalable TS) without trajectory optimization or DPP post-correction; the per-iteration sampler-construction cost is dominated by repeated $\mathcal{O}(m^2)$ SVGP variance evaluations and is independent of $B$. Second, the density depends on $\alpha_t$ only through point-wise evaluations: B3O is therefore compatible with any acquisition that already relies on marginal predictives---a property of the chosen acquisition rather than of B3O itself. Third, we are moving the algorithmic difficulty from optimization to sampling, not eliminating it: in continuous, smooth, moderate-dimensional domains a global density estimator such as recursive partitioning (\texttt{DEFER}~\citep{bodin2021black}) is well-suited but degrades as $d$ grows, whereas classical MCMC schemes such as Metropolis--Hastings \citep{hastings1970monte, chib1995understanding} or Langevin dynamics \citep{roberts1996exponential} mix slowly on the multimodal acquisition surfaces (Appendix~\ref{app:sampler_ablation}).

\paragraph{Versatility.} The framework is versatile in two distinct ways. (a) Because B3O leaves the surrogate, the acquisition, and the BO loop unchanged, any acquisition designed for a particular regime---e.g., expected hypervolume improvement~\citep{couckuyt2014fast} for multi-objective BO or constrained EI~\citep{gardner2014bayesian} for constrained BO---can be reused without modification. (b) In settings where optimization is intrinsically hard, most prominently mixed continuous--discrete search spaces, sampling can be easier than maximization, so one simply swaps the continuous sampler for a discrete- or mixed-space sampler (e.g., Metropolis--Hastings with a composite proposal) without touching the rest of the pipeline.

As summarized in Table~\ref{tab:comparison_table}, B3O thus delivers a single algorithmic recipe that retains the linear $\mathcal{O}(B)$ scaling of scalable TS while supporting heterogeneous problem types via either acquisition or sampler substitution.

\section{Regret Analysis}\label{sec:regret_analysis}

This section establishes a finite-time cumulative regret bound for B3O paired with the UCB acquisition (\eqref{eq:ucb}); we refer to this instantiation as B3O-UCB. The analysis is for the sequential setting ($B=1$), where each iteration draws a single query from the Boltzmann density rather than maximizing the acquisition; the goal is to show that this sampling step does not meaningfully harm regret compared with greedy maximization. We assume exact sampling from the true Boltzmann distribution $p_t(\mathbf{x})$, together with mild regularity of the input space and a quadratic lower bound on the acquisition curvature near its maximum (Assumptions~\ref{ass:convexity}--\ref{ass:quadratic_decay}, Appendix~\ref{app:proofs}). We do not analyze approximation errors introduced by the sampler or surrogate, nor do we formally extend the bound to general $B$; our empirical evaluations in Section~\ref{sec:experiments} show that practical global samplers approximate $p_t$ well enough at large $B$ to preserve the qualitative behaviour predicted by the sequential analysis.

\paragraph{Proof outline.} The analysis follows \citet{srinivas2010gaussian}'s GP-UCB framework; the only modification handles the gap $\alpha_t(\mathbf{x}_t^*) - \alpha_t(\mathbf{x}_t)$ introduced by sampling rather than maximizing. The Boltzmann density's exponential form admits a Laplace-style argument (Theorem~\ref{thm:prob_bound}, Appendix~\ref{app:proofs}) bounding this gap by $1/\sqrt{t}$ with high probability under an appropriate $\lambda_t$ schedule. Summing the per-step penalty over $T$ contributes an $O(\sqrt{T})$ term to the GP-UCB bound.

\begin{theorem}\label{thm:regret_bound}
In the sequential setting ($B=1$), with high probability, with inverse temperature $\lambda_{t}=O(\sqrt{t}\log t)$, the cumulative regret of B3O-UCB is upper-bounded by
\begin{equation}
    O\left(\sqrt{T}+\sqrt{C_{1}T\beta_{T}\gamma_{T}}\right).
\end{equation}
\end{theorem}

Compared with GP-UCB's $O(\sqrt{C_{1}T\beta_{T}\gamma_{T}})$ \citep{srinivas2010gaussian}, B3O-UCB preserves the standard rate up to a negligible additive term: replacing acquisition maximization with Boltzmann sampling does not meaningfully harm regret in the sequential setting. The argument extends naturally to constrained acquisitions (e.g., constrained EI) provided Assumption~\ref{ass:quadratic_decay} holds locally near the constrained optimum; a formal extension to $B>1$ and to constrained acquisitions is left to future work.

\section{Experiments}\label{sec:experiments}
 We first demonstrate large-batch efficiency on multimodal synthetic benchmarks with significant observation noise, comparing against similar large-batch BO algorithms. We then demonstrate robustness to problem type on a real-world multi-objective problem and a mixed-search-space problem. Our implementation uses \texttt{trieste}~\citep{berkeley2021uri} with \texttt{gpflow}~\citep{matthews2017gpflow} for sparse GP modeling. All baselines---batch Expected Improvement ($q$-EI), Local Penalization, scalable Thompson Sampling, sequential Expected Improvement, and sequential Expected Hypervolume Improvement (EHVI)---use default \texttt{trieste} implementations.

\subsection{Synthetic Benchmarks}\label{sec:syn_bench}
We evaluate B3O against two widely used baselines: Sequential Expected Improvement~\citep{jones1998efficient}, a per-sample-efficiency ceiling despite its high cost, and TS~\citep{vakili2021scalable}, the state-of-the-art for large parallel resources. We evaluate B3O with both UCB and Log Expected Improvement (LogEI)~\citep{ament2023unexpected} acquisitions, under constant temperature and the schedule from Section~\ref{sec:regret_analysis}.

We adopt two regimes for $\lambda_t$. The scheduled regime sets $\lambda_t = \lambda_0 \sqrt{t} \log t$ following Section~\ref{sec:regret_analysis}: low $\lambda$ early encourages diversity, growing over time to concentrate on high-acquisition regions. The constant regime fixes $\lambda$ throughout. Because UCB and LogEI have different acquisition magnitudes, the inverse temperature is selected per-function and per-variant from the ablation in Appendix~\ref{app:lambda_ablation}; the values used in the main-text figure are listed in Table~\ref{tab:per_function_lambda}. Suffixes \mbox{\textbf{-c}} and \mbox{\textbf{-s}} denote constant and scheduled variants (e.g., B3O-UCB-s).

We test on three standard BO benchmarks: Shekel-4D, Ackley-5D, and Hartmann-6D. All methods use SVGP surrogates with a Mat\'ern-5/2 kernel; the \texttt{DEFER} sampler~\citep{bodin2021black} generates samples from the Boltzmann distribution. All algorithms are initialised with 100 random points. Batch methods (B3O, TS, Random Search) then use $B=100$ for $T=50$ iterations ($5{,}000$ further queries); Seq-EI is run for $750$ sequential steps with per-evaluation posterior updates as a high-fidelity baseline. Results are averaged over $30$ independent trials.

\begin{figure}[h!]
    \centering
    \includegraphics[width=0.9\linewidth]{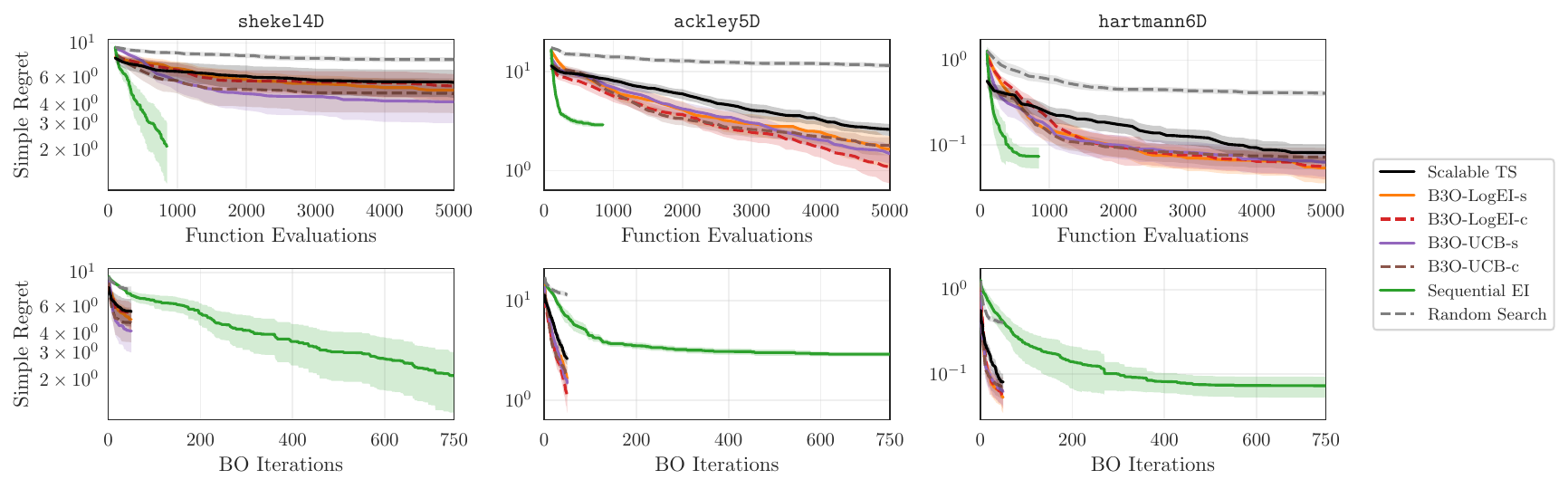}
    \caption{Simple regret on Shekel (4D, left), Ackley (5D, middle), and Hartmann (6D, right). The curves illustrate the performance of B3O using UCB and LogEI with both constant and scheduled temperatures, compared against scalable TS, Sequential Expected Improvement, and Random Search. B3O with LogEI and a constant inverse temperature consistently achieves the lowest regret amongst batch methods on Ackley and Hartmann, while on Shekel the scheduled UCB variant is most effective at escaping the function's deep local minima.}
    \label{fig:synthetic_benchmarking}
\end{figure}

The experimental results in Figure~\ref{fig:synthetic_benchmarking} highlight the efficacy of Boltzmann exploration in large-batch settings. All four B3O variants achieve lower final regret than the scalable TS baseline on every benchmark, and all batch methods clearly outperform Random Search. On Ackley and Hartmann, the batch methods match or outperform Sequential EI, demonstrating efficient use of parallel resources, and LogEI yields the strongest performance among the B3O configurations. Shekel is more challenging for batch methods, Sequential EI achieves the lowest final regret, but among the parallel methods B3O-UCB-s makes the most consistent progress towards the global optimum. Crucially, a constant temperature suffices to outperform TS across all three benchmarks, indicating the algorithm is robust and broadly applicable; practitioners can use a fixed inverse temperature for competitive performance without defining a custom schedule.

To evaluate B3O on smaller batch sizes, allowing comparison to existing small-batch algorithms, we conducted an additional study with $B=5$ on Shekel-4d, Ackley-5D and Hartmann-6D. Figure~\ref{fig:small_batch_investigation} compares B3O (both scheduled and constant temperature regimes) against several batch BO baselines including Thompson Sampling (TS), Local Penalization, and Batch EI ($q$-EI). 

B3O remains competitive in the small-batch regime, particularly on Ackley-5D and Hartmann-6D, where every B3O variant matches or beats $q$-EI and Local Penalization. $q$-EI and Local Penalization underperform on these multimodal benchmarks: the former's $B\cdot d$-dimensional Monte Carlo inner optimization is challenging on noisy SVGP marginals, and the latter's repulsion struggles with the dense local optima of Ackley and Hartmann. Shekel-4D is harder for the parallel methods at $B=5$:  B3O-UCB-s attains the lowest final regret, with Local Penlaization the next-best. B3O-UCB-c is notably weaker than its scheduled counterpart, since a fixed $\lambda$ for UCB cannot adapt to the changing scale as the surrogate improves. The constant LogEI variant remains competitive, indicating that Boltzmann exploration is effective across batch sizes without extensive retuning of $\lambda$.

\begin{figure}[h]
    \centering
    \includegraphics[width=1\linewidth]{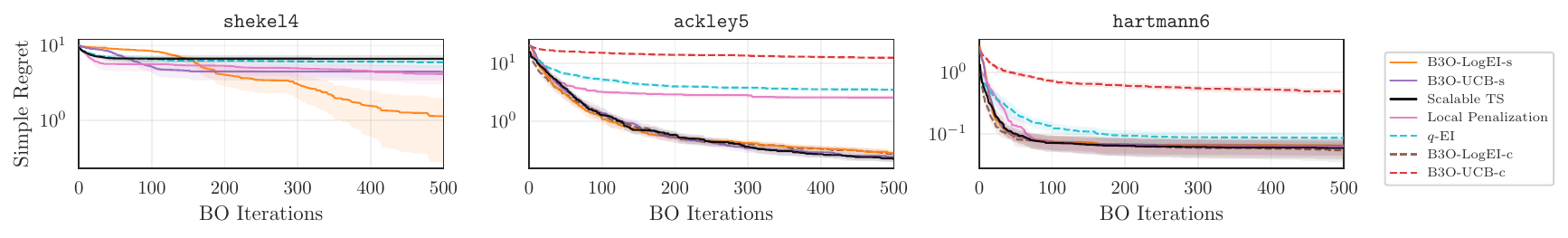}
    \caption{Comparison of B3O variants against batch BO baselines with a small batch size of $B=5$ on Shekel-4D, Ackley-5D and Hartmann-6D benchmarks.}
    \label{fig:small_batch_investigation}
\end{figure}

\subsection{Multi-Objective Battery Electrode Design}

We consider the multi-objective optimization of a lithium-ion battery design, aiming to simultaneously maximize specific energy and specific power~\citep{thebelt2022multi}. The underlying physics are simulated using the Single Particle Model with electrolyte effects via the \texttt{PyBaMM} framework~\citep{sulzer2021python}. The optimization tunes the microstructural and geometric parameters of the cell to manage the trade-offs between capacity and ion transport resistance. The nine decision variables include the porosity, active material volume fraction, and particle radius for both the anode and cathode, alongside the electrode thickness scaling factors. The B3O framework is adapted to this setting by changing the acquisition function to the Expected Hypervolume Improvement~\citep{couckuyt2014fast}.

We compare B3O against two baselines: the Non-dominated Sorting Genetic Algorithm II (NSGA-II)~\citep{deb2002fast}, run for 10,000 iterations to approximate the reference Pareto front, and sequential Expected Hypervolume Improvement (EHVI), run for 600 iterations. The B3O variants (constant and scheduled) are run for 12 iterations with a batch size of $B=50$, matching the 600-function-evaluation budget of the sequential baseline.

\begin{figure}[h]
    \centering
    \includegraphics[width=0.8\linewidth]{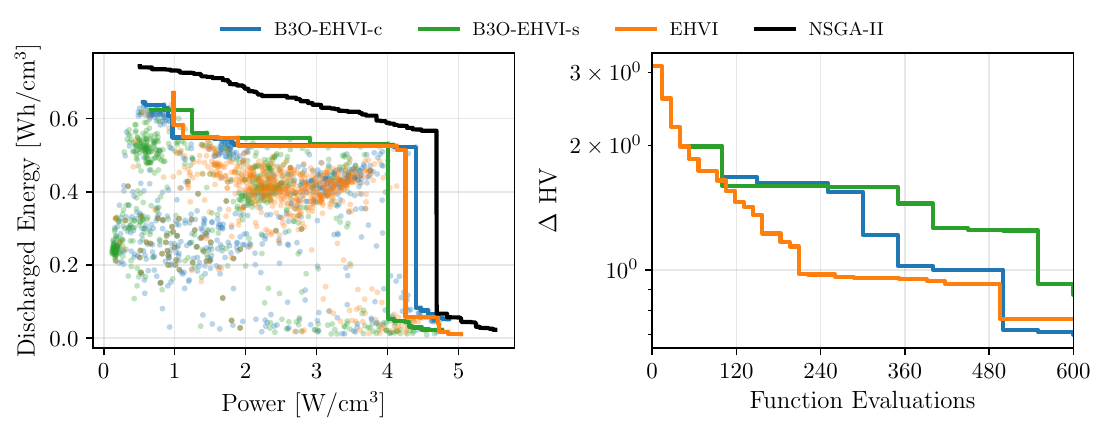}
    \caption{Optimization results for the battery electrode design. Left: Comparison of Pareto fronts identified by NSGA-II (reference), Sequential EHVI, and B3O variants. Right: Hypervolume difference relative to the NSGA-II reference, plotted against total function evaluations (sequential EHVI: $600$ sequential evaluations; B3O: $12$ iterations of $B=50$). B3O achieves similar performance to the sequential baseline in only 12 batch iterations, with the constant-temperature variant yielding superior diversity and final hypervolume.}
    \label{fig:battery_results}
\end{figure}

Figure~\ref{fig:battery_results} (left) illustrates the resulting Pareto fronts: B3O successfully identifies a diverse set of optimal designs that approximate the NSGA-II reference front. Figure~\ref{fig:battery_results} (right) presents the hypervolume difference relative to the NSGA-II reference, demonstrating that B3O achieves a convergence profile comparable to sequential EHVI within the same evaluation budget. Notably, the constant-temperature variant outperforms the scheduled approach, suggesting that a fixed temperature better maintains batch diversity throughout the optimization process and thereby attains a higher final hypervolume.

\subsection{Mixed-Variable Optimization}
To further demonstrate the versatility of the B3O framework in real-world engineering contexts, we apply it to the mixed-variable optimization of a Formula E race car configuration. This problem requires exploring a hybrid search space composed of two discrete decision variables (gear ratio and maximum power setting) alongside four continuous variables (motor torque, drag area, rear downforce, and weight distribution). The objective is to minimize a combined lap-time and energy-use metric on the Shanghai race circuit, utilizing the simulation environment provided by~\citep{Heilmeier2019}. Taking advantage of B3O's sampler-agnostic architecture to accommodate the heterogeneity of the decision space, we simply swap the continuous sampler for a specialized Metropolis-Hastings (MH) sampler, leaving the core BO loop unchanged. This sampler employs distinct proposal mechanisms for the continuous and discrete subspaces (see Appendix~\ref{app:implementation} for implementation details).
\begin{figure}[t]
    \centering
    \includegraphics[width=0.9\linewidth]{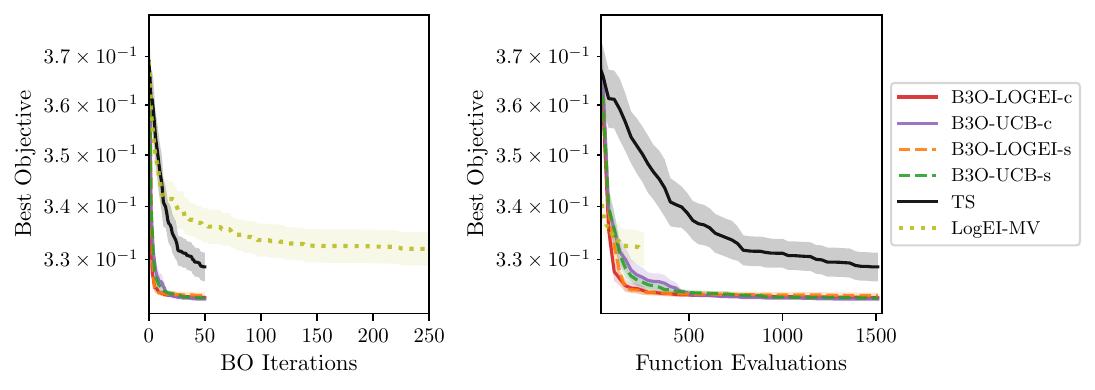}
    \caption{Best objective for the mixed-variable Formula E configuration problem. The left and right panels display performance with respect to optimization iterations and total function evaluations, respectively. B3O variants (LogEI and UCB, with constant and scheduled $\lambda$) are compared against scalable Thompson Sampling (TS) and a sequential mixed-variable LogEI baseline (LogEI-MV). The sequential LogEI-MV baseline terminates at $250$ evaluations, while B3O and TS continue to $1{,}530$ evaluations. Shaded regions represent 95\% confidence intervals.}
    \label{fig:discrete_benchmarking}
\end{figure}

Figure~\ref{fig:discrete_benchmarking} reports the best objective over 30 trials (we report best objective rather than simple regret since the simulator has no analytically known optimum). We compare against two baselines: scalable Thompson Sampling (TS)~\citep{vakili2021scalable} and a sequential mixed-variable LogEI baseline (LogEI-MV) implemented in \texttt{trieste}~\citep{berkeley2021uri}, which random-searches the discrete variables and then gradient-optimizes the continuous ones. All four B3O variants reach a lower final objective than either baseline. Per BO iteration (left), B3O converges in under $20$ iterations while TS needs several hundred to reach a comparable objective; per function evaluation (right), B3O dominates throughout, with TS reaching the B3O plateau only late in the budget, consistent with the diversity collapse of trajectory-based sampling on the hybrid landscape, and LogEI-MV exhausting its sequential budget first.

\section{Conclusion and Future Work}

We introduced B3O, a scalable Batch BO framework that reformulates batch generation as sampling from a Boltzmann density on the marginal acquisition. It requires only a point-wise acquisition and a sampler, leaves the surrogate and BO loop unchanged, and circumvents the variance starvation that limits trajectory-based scalable TS. We proved that this preserves the standard GP-UCB regret rate up to a negligible additive penalty, with strong empirical performance across synthetic benchmarks, multi-objective battery design, and mixed-variable race-car configuration. The formulation also composes naturally with other BO building blocks, e.g.\ trust-region methods \citep{eriksson2019scalable, diouane2023trego}.

The regret analysis (Section~\ref{sec:regret_analysis}) assumes exact sampling and a quadratic lower bound on the acquisition (Assumption~\ref{ass:quadratic_decay}); empirical results suggest the conclusions are robust under sampler approximation. Performance is sensitive to $\lambda$, whose magnitude depends on the acquisition's natural scale (Appendix~\ref{app:lambda_ablation}), motivating an adaptive scheme. Sampler choice also matters, local MCMC mixes poorly at large $B$ and recursive partitioning suits moderate dimensionality (Appendix~\ref{app:sampler_ablation}, and extending to higher-dimensional and structured spaces is a natural next step.

\bibliographystyle{unsrtnat}  % Choose a style (unsrt, plain, alpha, etc.)
\bibliography{references}

\newpage
\appendix
\section{Detailed Proofs for Boltzmann Exploration}\label{app:proofs}

In this section, we provide the proofs for the theoretical guarantees of the Boltzmann exploration strategy. These proofs rely on the assumptions of convexity of the input space $\mathcal{X}$ and the quadratic lower bound of the acquisition function $\alpha_t$.

\begin{assumption}
\label{ass:convexity}
The input space $\mathcal{X}$ is convex and compact, and without loss of generality, the integral $\int_{\mathcal{X}}d\mathbf{x}$ is assumed to equal 1.
\end{assumption}

Additionally, a lower bound on the acquisition function is required to characterize its behavior near the optimum.

\begin{assumption}
\label{ass:quadratic_decay}
For all $t\in[T]$, there exists a positive-semidefinite matrix $A_{t}$ such that the inequality
\begin{equation}
    \alpha_{t}(\mathbf{x})\ge\alpha_{t}(\mathbf{x}_{t}^{*})-\frac{1}{2}(\mathbf{x}-\mathbf{x}_{t}^{*})^{\top}A_{t}(\mathbf{x}-\mathbf{x}_{t}^{*})
\end{equation}
holds for all $\mathbf{x}\in\mathcal{X}$, where $\mathbf{x}_{t}^{*}=\arg \max_{\mathbf{x}\in\mathcal{X}}\alpha_{t}(\mathbf{x})$.
\end{assumption}

Assumption \ref{ass:quadratic_decay} holds for most continuous and smooth functions, such as when $\alpha_{t}(\mathbf{x})$ is Lipschitz continuous or concave, in which case $A_{t}$ can be set as the negative Hessian. Under these conditions, the probability of obtaining samples with small acquisition values can be bounded.

\begin{theorem}
\label{thm:prob_bound}
Let $\mathbf{x}_{t}$ be a sample from the Boltzmann distribution $p_{t}(\mathbf{x})\propto \exp(\lambda_{t}\alpha_{t}(\mathbf{x}))$. Under Assumptions \ref{ass:convexity} and \ref{ass:quadratic_decay}, the probability bound
\begin{equation}
    \mathbb{P}[\alpha_{t}(\mathbf{x}_{t}^{*})-\alpha_{t}(\mathbf{x}_{t})\ge\epsilon]\le \exp\left(-\lambda_{t}\epsilon+\frac{d}{2}\log \lambda_{t}-\log V_{t}\right)
\end{equation}
holds for any $\lambda_{t}\ge1$, where $V_{t}=\int_{\mathcal{X}}\exp(-\frac{1}{2}(\mathbf{x}-\mathbf{x}_{t}^{*})^{\top}A_{t}(\mathbf{x}-\mathbf{x}_{t}^{*}))d\mathbf{x}$.
\end{theorem}

\subsection{Proof of Theorem \ref{thm:prob_bound}}\label{det_proof:1}

\begin{proof}
Let $S_\epsilon = \{\mathbf{x} \in \mathcal{X} \mid \alpha_t(\mathbf{x}) \leq \alpha_t(\mathbf{x}_t^*) - \epsilon\}$ denote the sub-optimal set. The probability of sampling a point $\mathbf{x}_t$ from the Boltzmann distribution $p_{t}(\mathbf{x})\propto \exp(\lambda_{t}\alpha_{t}(\mathbf{x}))$ such that the acquisition value is at least $\epsilon$ away from the maximum,
\begin{equation}
\mathbb{P}[\alpha_t(\mathbf{x}_t^*) - \alpha_t(\mathbf{x}_t) \geq \epsilon] = \frac{\int_{S_\epsilon} \exp(\lambda_t \alpha_t(\mathbf{x}))\,\mathrm{d}\mathbf{x}}{\int_{\mathcal{X}} \exp(\lambda_t \alpha_t(\mathbf{x}))\,\mathrm{d}\mathbf{x}}.
\end{equation}
By dividing both the numerator and denominator by $\exp(\lambda_t \alpha_t(\mathbf{x}_t^*))$ and defining $\Delta_t(\mathbf{x}) = \alpha_t(\mathbf{x}) - \alpha_t(\mathbf{x}_t^*)$, we obtain
\begin{equation}
\mathbb{P}[\alpha_t(\mathbf{x}_t^*) - \alpha_t(\mathbf{x}_t) \geq \epsilon] = \frac{\int_{S_\epsilon} \exp(\lambda_t \Delta_t(\mathbf{x}))\,\mathrm{d}\mathbf{x}}{\int_{\mathcal{X}} \exp(\lambda_t \Delta_t(\mathbf{x}))\,\mathrm{d}\mathbf{x}}.
\end{equation}
By the definition of $S_\epsilon$, we have $\Delta_t(\mathbf{x}) \leq -\epsilon$ for all $\mathbf{x} \in S_\epsilon$. Thus, the numerator is bounded as
\begin{equation}
\int_{S_\epsilon} \exp(\lambda_t \Delta_t(\mathbf{x}))\,\mathrm{d}\mathbf{x} \leq \int_{S_\epsilon} \exp(-\lambda_t \epsilon)\,\mathrm{d}\mathbf{x} \leq \exp(-\lambda_t \epsilon),
\end{equation}
where the last inequality follows from Assumption \ref{ass:convexity} (specifically $\int_{\mathcal{X}}\,\mathrm{d}\mathbf{x} = 1$). For the denominator, using Assumption \ref{ass:quadratic_decay},
\begin{equation}
\int_{\mathcal{X}} \exp(\lambda_t \Delta_t(\mathbf{x}))\,\mathrm{d}\mathbf{x} \geq \int_{\mathcal{X}} \exp\left(-\frac{\lambda_t}{2}(\mathbf{x} - \mathbf{x}_t^*)^\top A_t (\mathbf{x} - \mathbf{x}_t^*)\right)\,\mathrm{d}\mathbf{x}.
\end{equation}
Applying a change of variables $\mathbf{y} = \lambda_t^{1/2}(\mathbf{x} - \mathbf{x}_t^*)$ with Jacobian $|\mathrm{d}\mathbf{x}/\mathrm{d}\mathbf{y}| = \lambda_t^{-d/2}$
\begin{equation}
\int_{\mathcal{X}} \exp(\lambda_t \Delta_t(\mathbf{x}))\,\mathrm{d}\mathbf{x} \geq \lambda_t^{-d/2} \int_{\{\lambda_t^{1/2}(\mathbf{x}-\mathbf{x}_t^*) \mid \mathbf{x} \in \mathcal{X}\}} \exp\left(-\frac{1}{2}\mathbf{y}^\top A_t \mathbf{y}\right)\,\mathrm{d}\mathbf{y}.
\end{equation}
Given $\lambda_t \geq 1$ and the convexity of $\mathcal{X}$, the integration domain expands as $\lambda_t$ increases,
\begin{equation}
\int_{\mathcal{X}} \exp(\lambda_t \Delta_t(\mathbf{x}))\,\mathrm{d}\mathbf{x} \geq \lambda_t^{-d/2} \int_{\{(\mathbf{x}-\mathbf{x}_t^*) \mid \mathbf{x} \in \mathcal{X}\}} \exp\left(-\frac{1}{2}\mathbf{y}^\top A_t \mathbf{y}\right)\,\mathrm{d}\mathbf{y} = \lambda_t^{-d/2} V_t.
\end{equation}
where $V_t = \int_{\mathcal{X}} \exp(-\frac{1}{2}(\mathbf{x}-\mathbf{x}_{t}^{*})^{\top}A_{t}(\mathbf{x}-\mathbf{x}_{t}^{*}))\,\mathrm{d}\mathbf{x}$. Combining these bounds yields the result for Theorem \ref{thm:prob_bound},
\begin{equation}
\mathbb{P}[\alpha_t(\mathbf{x}_t^*) - \alpha_t(\mathbf{x}_t) \geq \epsilon] \leq \exp\left(-\lambda_t \epsilon + \frac{d}{2} \log \lambda_t - \log V_t\right).
\end{equation}
\end{proof}

\subsection{Proof of Theorem \ref{thm:regret_bound}}\label{det_proof:2}

\begin{proof}
We utilize the GP-UCB regret analysis framework. Let $\mathbf{x}^* = \arg \max f(\mathbf{x})$ and $\mathbf{x}_t^* = \arg \max \alpha_t(\mathbf{x})$. We first show that, with an appropriate inverse temperature schedule, the per-step sampling gap $\alpha_t(\mathbf{x}_t^*) - \alpha_t(\mathbf{x}_t)$ is at most $1/\sqrt{t}$ for all $t \in [T]$ with high probability.

Applying Theorem~\ref{thm:prob_bound} with $\epsilon = 1/\sqrt{t}$, the gap exceeds $1/\sqrt{t}$ with probability at most $\exp(-\lambda_t/\sqrt{t} + (d/2)\log\lambda_t - \log V_t)$. To make a union bound over $t \in [T]$ yield total failure probability at most $\delta$, it suffices that for every $t$,
\begin{equation}\label{eq:lambda_condition}
    \frac{\lambda_t}{\sqrt{t}} \;\geq\; \log\frac{\pi^2 t^2}{6\delta} + \frac{d}{2}\log\lambda_t - \log V_t,
\end{equation}
since $\sum_{t\geq 1} 6\delta/(\pi^2 t^2) \leq \delta$. We split \eqref{eq:lambda_condition} into two sufficient conditions, each absorbing half of the right-hand side:
\begin{equation}\label{eq:lambda_split}
    \frac{\lambda_t}{2\sqrt{t}} \;\geq\; \log\frac{\pi^2 t^2}{6\delta},
    \qquad
    \frac{\lambda_t}{2\sqrt{t}} \;\geq\; \frac{d}{2}\log\lambda_t - \log V_t.
\end{equation}
The first inequality is satisfied by $\lambda_t \geq 2\sqrt{t}\log(\pi^2 t^2/(6\delta))$. For the second, let $\bar\lambda_t$ denote the smallest value satisfying $\bar\lambda_t \geq \sqrt{t}\,(d\log\bar\lambda_t - 2\log V_t)$. Such a $\bar\lambda_t$ exists because the left-hand side grows linearly in $\bar\lambda_t$ while the right-hand side grows only logarithmically, and $\bar\lambda_t = O(\sqrt{t}\log t)$ provided $V_t$ is bounded below by a positive constant (which holds whenever $A_t$ has bounded eigenvalues, true in particular when $\alpha_t$ is Lipschitz; cf.\ the remark following Assumption~\ref{ass:quadratic_decay}). Setting
\begin{equation}\label{eq:lambda_schedule}
    \lambda_t \;=\; \max\!\left(\bar\lambda_t,\; 2\sqrt{t}\log\frac{\pi^2 t^2}{6\delta}\right) \;=\; O(\sqrt{t}\log t)
\end{equation}
ensures both inequalities in \eqref{eq:lambda_split}, hence \eqref{eq:lambda_condition}, holds for all $t$. Note that $\bar\lambda_t$ here is a purely analytic quantity arising from the bound and is distinct from the user-chosen $\lambda_0$ in the practical schedule $\lambda_t = \lambda_0\sqrt{t}\log t$ used in the experiments. The union bound then gives
\begin{equation}\label{eq:gap_bound}
    \mathbb{P}\!\left[\forall t\in[T]:\; \alpha_t(\mathbf{x}_t^*) - \alpha_t(\mathbf{x}_t) \leq 1/\sqrt{t}\right] \;\geq\; 1 - \delta.
\end{equation}

Combining \eqref{eq:gap_bound} with the standard GP-UCB confidence event (which holds with probability at least $1-\delta$~\citep{srinivas2010gaussian}), with probability at least $1-2\delta$, the cumulative regret is bounded as
\begin{align}
\sum_{t=1}^{T} f(\mathbf{x}^*) - f(\mathbf{x}_t) &\leq \sum_{t=1}^{T} \mu_t(\mathbf{x}^*) + \beta_t^{1/2}\sigma_t(\mathbf{x}^*) - f(\mathbf{x}_t) \\
&\leq \sum_{t=1}^{T} \mu_t(\mathbf{x}_t^*) + \beta_t^{1/2}\sigma_t(\mathbf{x}_t^*) - f(\mathbf{x}_t) \\
&\leq \sum_{t=1}^{T} \mu_t(\mathbf{x}_t) + \beta_t^{1/2}\sigma_t(\mathbf{x}_t) + 1/\sqrt{t} - f(\mathbf{x}_t) \\
&\leq 1 + 2\sqrt{T} + \sum_{t=1}^{T} 2\beta_t^{1/2}\sigma_t(\mathbf{x}_t).
\end{align}
Using the fact that $\sum_{t=1}^{T} \sigma_t(\mathbf{x}_t) \leq \sqrt{C_1 T \beta_T \gamma_T}$, the total regret is
\begin{equation}
R_T \leq 1 + 2\sqrt{T} + O(\sqrt{C_1 T \beta_T \gamma_T}).
\end{equation}
This confirms that Boltzmann exploration adds only a negligible additional term to the regret compared to the original GP-UCB analysis~\citep{srinivas2010gaussian}.
\end{proof}
\section{Experiment Details}
Section~\ref{sec:syn_bench} presented results from three synthetic benchmark functions, namely Shekel4, Ackley5 and Hartmann6, all of which were contaminated with Gaussian noise of variance $0.5$. The forms of these functions are given below.

\paragraph{Shekel Function}
The 4-dimensional Shekel function is a benchmark with $m=10$ Gaussian-like wells of varying depth and width, defined on the hyper-rectangle $\mathbf{x} \in [0, 10]^4$:
\begin{equation}
    f(\mathbf{x}) = \sum_{i=1}^{m} \frac{1}{c_i + \sum_{j=1}^{4} (x_j - C_{ji})^2},
\end{equation}
with depth vector $\mathbf{c} = \tfrac{1}{10}(1,2,2,4,4,6,3,7,5,5)^\top$ and well-centre matrix
\begin{equation}
    \mathbf{C} = \begin{pmatrix}
    4 & 1 & 8 & 6 & 3 & 2 & 5 & 8 & 6 & 7 \\
    4 & 1 & 8 & 6 & 7 & 9 & 5 & 1 & 2 & 3.6 \\
    4 & 1 & 8 & 6 & 3 & 2 & 3 & 8 & 6 & 7 \\
    4 & 1 & 8 & 6 & 7 & 9 & 3 & 1 & 2 & 3.6
    \end{pmatrix}.
\end{equation}
The deepest well, located near $\mathbf{x}^* \approx (4,4,4,4)$, contains the global maximum, while the remaining wells form competing local optima.

\paragraph{Ackley Function}
The 5-dimensional Ackley function is a widely used benchmark characterized by a nearly flat outer region modulated by cosine waves and a deep central hole, creating a landscape with thousands of local minima surrounding a single global optimum. This geometry poses a significant challenge for exploration. Defined on the domain $\mathbf{x} \in [-2, 1]^5$, the function is given by:
\begin{equation}
    f(\mathbf{x}) = -20 \exp\left(-0.2 \sqrt{\frac{1}{d} \sum_{i=1}^d x_i^2}\right) - \exp\left(\frac{1}{d} \sum_{i=1}^d \cos(2\pi x_i)\right) + 20 + \exp(1)
\end{equation}
where $d=5$. The global maximum for the negated minimization problem is located at the origin.

\paragraph{Hartmann 6 Function}
The Hartmann 6-dimensional function is defined over the unit hypercube $\mathbf{x} \in [0, 1]^6$. It possesses six local maxima, making it a standard test for multimodal optimization. The function is analytically defined as:
\begin{equation}
    f(\mathbf{x}) = \sum_{i=1}^4 \alpha_i \exp\left(-\sum_{j=1}^6 A_{ij}(x_j - P_{ij})^2\right)
\end{equation}
The coefficients are given by the vector $\boldsymbol{\alpha} = (1.0, 1.2, 3.0, 3.2)^\top$. The matrix $\mathbf{A}$ is defined as
\begin{equation}
    \mathbf{A} = \begin{pmatrix}
    10 & 3 & 17 & 3.5 & 1.7 & 8 \\
    0.05 & 10 & 17 & 0.1 & 8 & 14 \\
    3 & 3.5 & 1.7 & 10 & 17 & 8 \\
    17 & 8 & 0.05 & 10 & 0.1 & 14
    \end{pmatrix},
\end{equation}
and the matrix $\mathbf{P}$ is given by
\begin{equation}
    \mathbf{P} = 10^{-4} \times \begin{pmatrix}
    1312 & 1696 & 5569 & 124 & 8283 & 5886 \\
    2329 & 4135 & 8307 & 3736 & 1004 & 9991 \\
    2348 & 1451 & 3522 & 2883 & 3047 & 6650 \\
    4047 & 8828 & 8732 & 5743 & 1091 & 381
    \end{pmatrix}.
\end{equation}

\subsection{Implementation Details}\label{app:implementation}

\paragraph{Synthetic Experiment} For the synthetic benchmarks in Section~\ref{sec:experiments}, we implement the SVGPs in \texttt{trieste}~\citep{berkeley2021uri} using a likelihood variance of 0.01, 500 inducing points, and 1000 random Fourier features, which follows the implementation used in~\citep{vakili2021scalable}. The acquisition functions are optimized using L-BFGS~\citep{liu1989limited} starting from the best of $2000d$ initial randomly sampled locations. All algorithms tested are initialized with 100 randomly chosen initial points. The synthetic benchmarking functions use a scaled search space of $[0,1]^d$, and their outputs are standardized using a large initial space-filling run.

\paragraph{Boltzmann Algorithm}
We assess the performance of B3O under both fixed and time-varying inverse temperature regimes. For the scheduled variants, we adopt the annealing rate $\lambda_t = \lambda_0 \sqrt{t} \log t$, consistent with the theoretical analysis in Section~\ref{sec:regret_analysis}. We select the inverse temperature per function and per B3O variant by minimum final regret in the lambda ablation of Appendix~\ref{app:lambda_ablation} (specific values listed in Table~\ref{tab:per_function_lambda}). The optimal $\lambda$ depends on both the acquisition scale and the geometry of the objective: scheduled LogEI uniformly favours small $\lambda_0=0.1$ on the smoother Ackley/Hartmann surfaces, while the multimodal Shekel landscape and constant UCB benefit from larger temperatures ($\lambda \in \{5, 10\}$).
\paragraph{Multi-Objective Battery Electrode Design}
The battery optimization problem targets a 9-dimensional continuous search space $\mathcal{X} = [0, 1]^9$ defining microstructural parameters, including porosity, active material volume fractions, and particle radii for both electrodes. To address the simultaneous maximization of specific power and energy, the Boltzmann energy landscape is defined via the Expected Hypervolume Improvement (EHVI) acquisition function.

Physical feasibility is ensured by linear constraints requiring the sum of porosity and active material volume fraction for both the anode and cathode to not exceed $0.95$. For the sequential EHVI baseline, these constraints are explicitly handled by the acquisition optimizer. In the B3O framework, the Boltzmann distribution is initially defined over the unconstrained hyperrectangle. Feasibility is subsequently enforced via rejection sampling: candidate points violating the linear constraints are discarded. This effectively reweights the energy function, assigning zero probability mass to the infeasible region while preserving the relative Boltzmann probabilities within the valid domain.

For the constant temperature regime, we set $\lambda = 1.0$, while the scheduled regime utilizes an initial $\lambda_0 = 10.0$. The optimization proceeds for $T=12$ iterations with a batch size of $B=50$, resulting in a total budget of 600 function evaluations.

\paragraph{Mixed-Variable Optimization}
The Formula E race car optimization problem, based on the simulation environment provided by~\citep{Heilmeier2019}, presents a mixed search space $\mathcal{X} = \mathcal{X}_c \times \mathcal{X}_d$. This space consists of four continuous parameters $\mathcal{X}_c \subseteq \mathbb{R}^4$ (torque, drag area, rear downforce, weight distribution) and two categorical parameters $\mathcal{X}_d$ (gear ratio and maximum power) with five choices each.

To handle this heterogeneity, we cannot rely on the recursive partitioning sampler used in the continuous experiments. Instead, we implement a specialized Metropolis-Hastings (MH) sampler to draw diverse batches from the Boltzmann distribution defined over the mixed domain. The sampler utilizes a composite proposal distribution $q(\mathbf{x}' | \mathbf{x})$. For the continuous subspace, we employ a random walk proposal with Gaussian noise:
\begin{equation}
    \mathbf{x}'_c = \mathbf{x}_c + \epsilon, \quad \epsilon \sim \mathcal{N}(0, \sigma^2 \mathbf{I}),
\end{equation}
where the step size $\sigma$ is adapted based on the scale of the bounds. For the discrete subspace, the proposal mechanism samples uniformly over the categories of the respective discrete variables. Due to the symmetric nature of these proposals, the MH acceptance probability simplifies to the ratio of the Boltzmann densities:
\begin{equation}
    A(\mathbf{x}, \mathbf{x}') = \min\left(1, \frac{\exp(\lambda_t \alpha_t(\mathbf{x}'))}{\exp(\lambda_t \alpha_t(\mathbf{x}))}\right).
\end{equation}

For the Boltzmann distribution parameters, we evaluate B3O using both scheduled and constant temperature regimes. In the scheduled setting, we utilize an initial inverse temperature of $\lambda_0 = 0.1$, while for the constant regime, we fix $\lambda = 1.0$. The surrogate model is an SVGP with a Mat\'ern-5/2 kernel. The discrete variables are handled via integer mapping within the kernel, while the sampler ensures that only valid discrete configurations are proposed for evaluation. The optimization is initialized with $N_{\text{init}}=30$ random points. We perform $50$ iterations with a batch size of $B=30$. The MH sampler runs for 1,000 burn-in steps with a thinning factor of 10 to ensure chain mixing and independence between samples within the batch. We evaluate B3O using both LogEI and UCB acquisition functions, comparing against a standard sequential Expected Improvement baseline that optimizes the continuous relaxation of the discrete variables.

\section{Additional Experiments}

\subsection{Lambda Ablation}\label{app:lambda_ablation}

To assess the robustness of the proposed framework, we conduct an ablation study on the inverse temperature parameter $\lambda$, which governs the diversity of the batch. We evaluated the algorithm's sensitivity across a range of values $\lambda \in \{0.1, 1, 5, 10\}$ for both constant and scheduled regimes on Shekel-4D, Ackley-5D and Hartmann-6D. Figure~\ref{fig:constant_lambda_ablation} and Figure~\ref{fig:scheduled_lambda_ablation} illustrate the performance impact of these configurations.

The results highlight that the constant temperature setup is generally robust, yielding competitive regret minimization across a broad range of hyperparameters. However, the efficacy of specific values remains tied to the scale of the acquisition function. For UCB, which typically exhibits smaller scales, lower constant values (e.g., $\lambda=0.1$) result in a distribution that is too uniform. This leads to excessive diversity, preventing the algorithm from effectively exploiting high-utility regions. Conversely, for LogEI in the scheduled setting ($\lambda_t = \lambda_0 \sqrt{t} \log t$), initializing with high values ($\lambda_0 \in \{5, 10\}$) causes the distribution to contract too rapidly. This mimics greedy maximization, restricting batch diversity and increasing the risk of getting trapped in local optima. These findings suggest that while a constant $\lambda$ is a safe default, the temperature range must be sufficient to maintain distinct modes in the acquisition landscape without flattening them into noise.

The main-text large-batch figure (Figure~\ref{fig:synthetic_benchmarking}) reports B3O with the per-function best $\lambda$ identified by this ablation, summarised in Table~\ref{tab:per_function_lambda}. The smoother Ackley-5D and Hartmann-6D benchmarks favour low scheduled LogEI initialisations ($\lambda_0=0.1$) and high UCB temperatures ($\lambda \in \{5, 10\}$), whereas the multimodal Shekel-4D landscape requires a moderate-to-high $\lambda$ across all four B3O variants ($\lambda \in \{1, 5\}$): too low a temperature flattens the acquisition modes into a near-uniform distribution, while too high a $\lambda_0$ on the scheduled UCB collapses early to a poor mode. We report these per-function values rather than a single global setting because the lambda sweep above documents the search space and shows that the choice in each cell is the highest-performing configuration available, not a hand-tuned outlier.

\begin{table}[h]
\centering
\caption{Best inverse temperature $\lambda$ per function for each B3O variant, selected by minimum final simple regret in the lambda ablation. These are the values used in the main-text large-batch figure (Figure~\ref{fig:synthetic_benchmarking}).}
\label{tab:per_function_lambda}
\begin{tabular}{@{} l c c c c @{}}
\toprule
Function & B3O-LogEI-s ($\lambda_0$) & B3O-LogEI-c ($\lambda$) & B3O-UCB-s ($\lambda_0$) & B3O-UCB-c ($\lambda$) \\
\midrule
\texttt{shekel4D}    & $5$   & $1$   & $1$  & $5$  \\
\texttt{ackley5D}    & $0.1$ & $1$   & $5$  & $10$ \\
\texttt{hartmann6D}  & $0.1$ & $0.1$ & $5$  & $10$ \\
\bottomrule
\end{tabular}
\end{table}

\begin{figure}[h]
    \centering
    \includegraphics[width=0.9\linewidth]{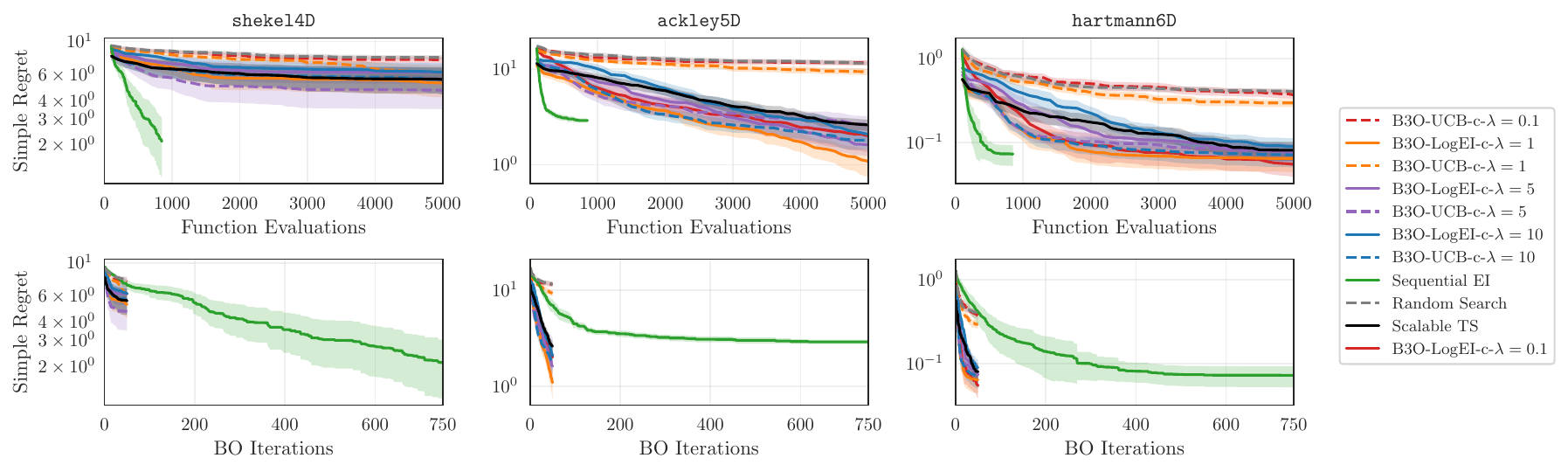}
    \caption{Ablation of \textbf{constant }inverse temperature $\lambda$ on simple regret. The constant setup proves robust, though UCB performance degrades at lower constant lambdas ($\lambda \in \{0.1, 1\}$), where the resulting high-temperature distribution induces excessive diversity within the batch.}
    \label{fig:constant_lambda_ablation}
\end{figure}

\begin{figure}[h]
    \centering
    \includegraphics[width=0.9\linewidth]{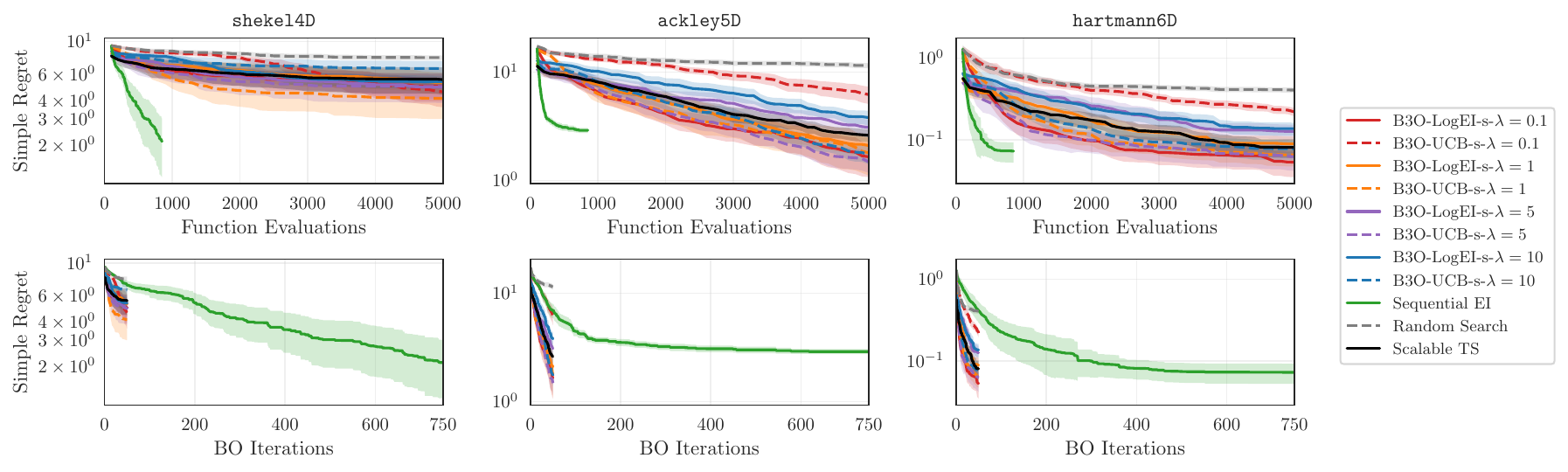}
    \caption{Ablation of the \textbf{initial lambda parameter} $\lambda_0$ for the time-varying temperature scheme. While generally effective, LogEI deteriorates with large initial values ($\lambda_0 \in \{5, 10\}$), where the high energy barrier restricts diversity and forces premature exploitation.}
    \label{fig:scheduled_lambda_ablation}
\end{figure}

%%%%%%%%%%
\section{Sampler Ablation}\label{app:sampler_ablation}

To evaluate the influence of the sampling mechanism on the efficiency of B3O, we conducted an ablation study across four distinct sampling strategies: Metropolis-Hastings (M-H)~\citep{hastings1970monte, chib1995understanding}, Langevin Dynamics~\citep{roberts1996exponential}, Discretized Monte Carlo (MC), and the Recursive Partitioning (\texttt{DEFER}) sampler~\citep{bodin2021black}. The M-H and Langevin Dynamics configurations evaluated here closely correspond to the families of local samplers employed by \citet{garcia2019fully, garciabarcos2025advanced} for acquisition-density sampling on continuous benchmarks at small batch sizes; this ablation therefore additionally illustrates how those sampler families behave on the multimodal large-batch acquisition surfaces targeted by B3O. To ensure consistency with our main results, we evaluate these samplers using the same inverse temperature configurations as the synthetic benchmarking: $\lambda=0.1$ for Scheduled LogEI (LogEI-s), $\lambda=5$ for Constant UCB (UCB-c), $\lambda=1$ for Constant LogEI (LogEI-c), and $\lambda=1$ for Scheduled UCB (UCB-s). Across both the Ackley-5D and Hartmann-6D benchmarks, the Recursive Partitioning sampler consistently achieves the lowest simple regret (Figure~\ref{fig:sampler_ucb_const}, Figure~\ref{fig:sampler_ucb_sched}), demonstrating superior global mixing compared to MCMC-based methods that often struggle with high energy barriers between isolated modes. While LogEI generally provides more stable convergence across all samplers (Figure~\ref{fig:sampler_logei_const}, Figure~\ref{fig:sampler_logei_sched}), the choice of sampler is particularly critical for UCB, where the Defer sampler effectively maps the acquisition landscape to capture significant probability mass in unexplored regions (Figure~\ref{fig:sampler_ucb_const}). Furthermore, the Defer sampler's piecewise constant approximation provides a more reliable global representation of the Boltzmann distribution, leading to tighter confidence intervals and more consistent batch diversity than local samplers like Langevin Dynamics, which show slower convergence even under optimized schedules (Figure~\ref{fig:sampler_ucb_sched}, Figure~\ref{fig:sampler_logei_sched}). This pattern is consistent with the small-batch regime in which these local samplers were originally deployed by \citet{garcia2019fully, garciabarcos2025advanced}: on the larger, multimodal acquisition surfaces studied here, the global coverage afforded by a recursive-partitioning sampler becomes a more decisive factor.
\begin{figure}
    \centering
    \includegraphics[width=0.9\linewidth]{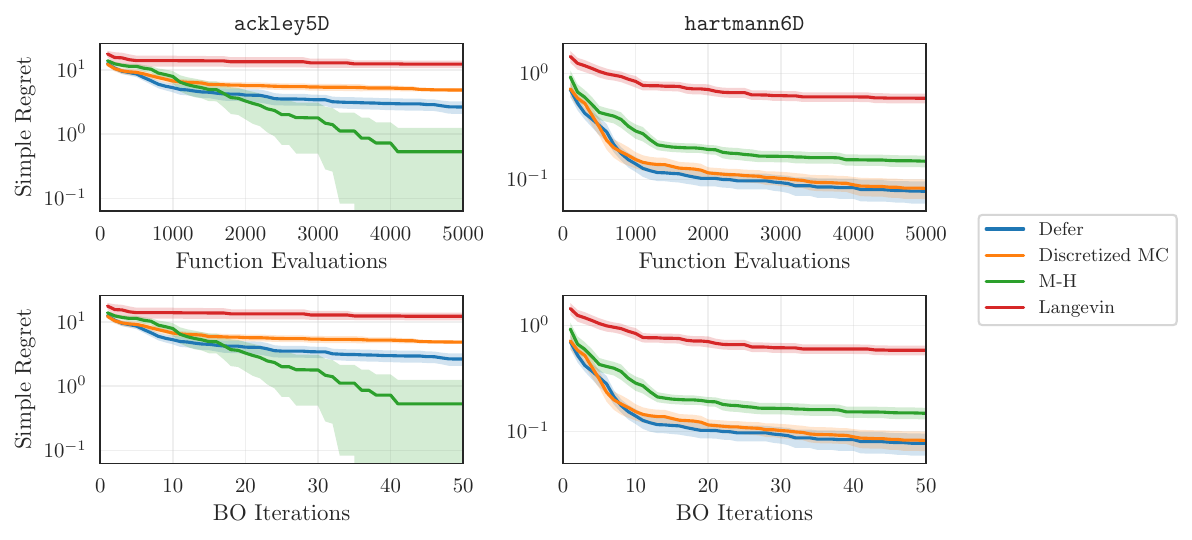}
    \caption{Performance comparison of sampling strategies using the UCB acquisition function with constant $\lambda$=5.}
    \label{fig:sampler_ucb_const}
\end{figure}

\begin{figure}
    \centering
    \includegraphics[width=0.9\linewidth]{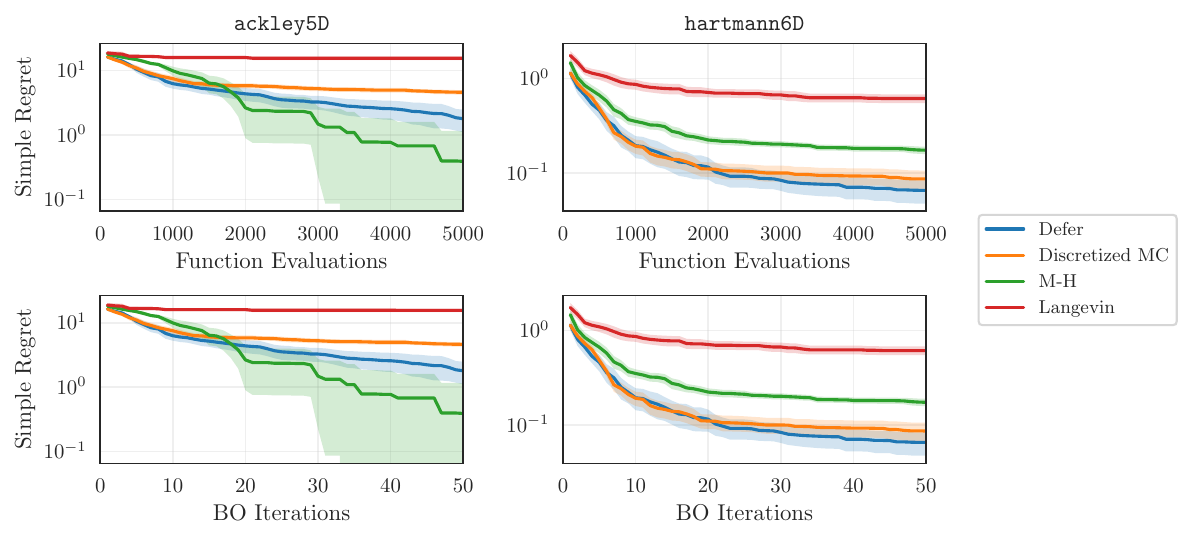}
    \caption{Performance comparison of sampling strategies using the UCB acquisition function with a scheduled inverse temperature starting from $\lambda_0=1$.}
    \label{fig:sampler_ucb_sched}
\end{figure}

\begin{figure}
    \centering
    \includegraphics[width=0.9\linewidth]{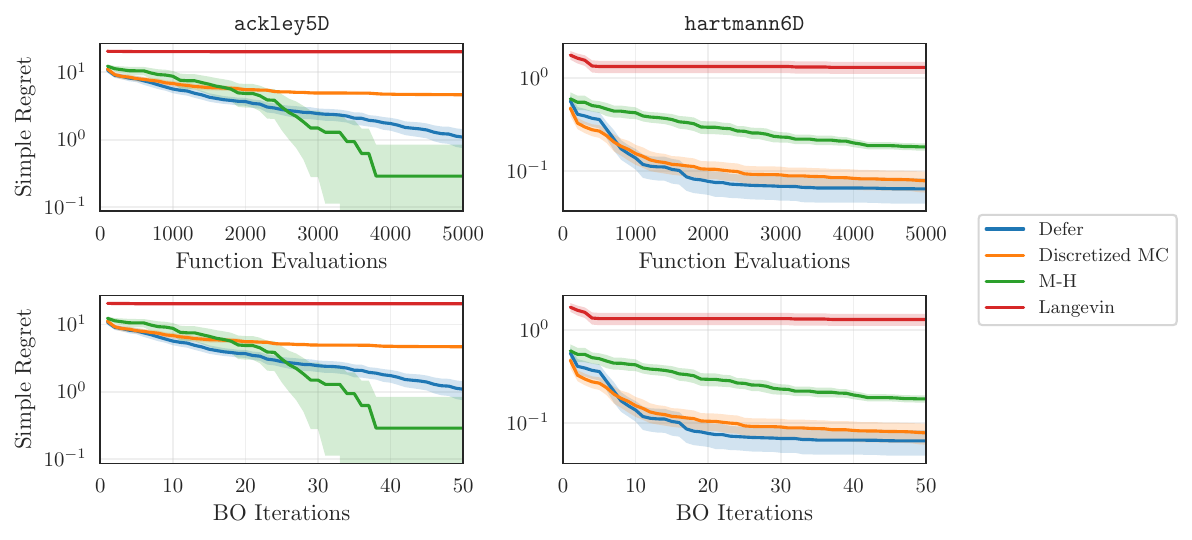}
    \caption{Performance comparison of sampling strategies using the LogEI acquisition function with constant $\lambda=1$.}
    \label{fig:sampler_logei_const}
\end{figure}

\begin{figure}
    \centering
    \includegraphics[width=0.9\linewidth]{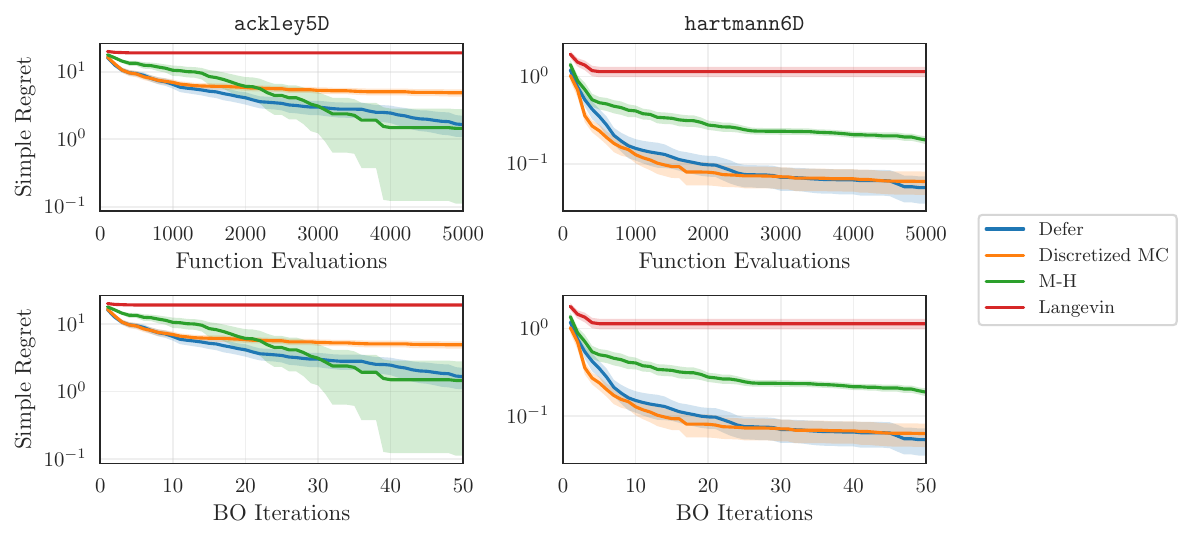}
    \caption{Performance comparison of sampling strategies using the LogEI acquisition function with a scheduled inverse temperature of $\lambda_0=0.1$.}
    \label{fig:sampler_logei_sched}
\end{figure}
\section{Batch Diversity Investigation}

To further characterize the behavior of B3O, we analyze the diversity of the generated batches throughout the optimization process. We define batch diversity as the average pairwise Euclidean distance between all points in a batch, $\text{div}(\mathcal{B}) = \frac{2}{B(B-1)} \sum_{i < j} \|x_i - x_j\|_2$. Figure~\ref{fig:diversity_benchmarking} illustrates this metric for the Shekel-4D, Ackley-5D, and Hartmann-6D benchmarks. We observe that B3O maintains significantly higher batch diversity than the scalable TS baseline across all configurations. Notably, the diversity of TS batches drops rapidly in early iterations, suggesting a tendency for trajectory samples to collapse toward a few dominant modes. In contrast, B3O configurations preserve a more exploration-oriented distribution, with the LogEI variants and scheduled UCB maintaining particularly consistent diversity levels. This sustained exploration likely contributes to the superior regret performance.

\begin{figure}[h!]
    \centering
    \includegraphics[width=0.9\linewidth]{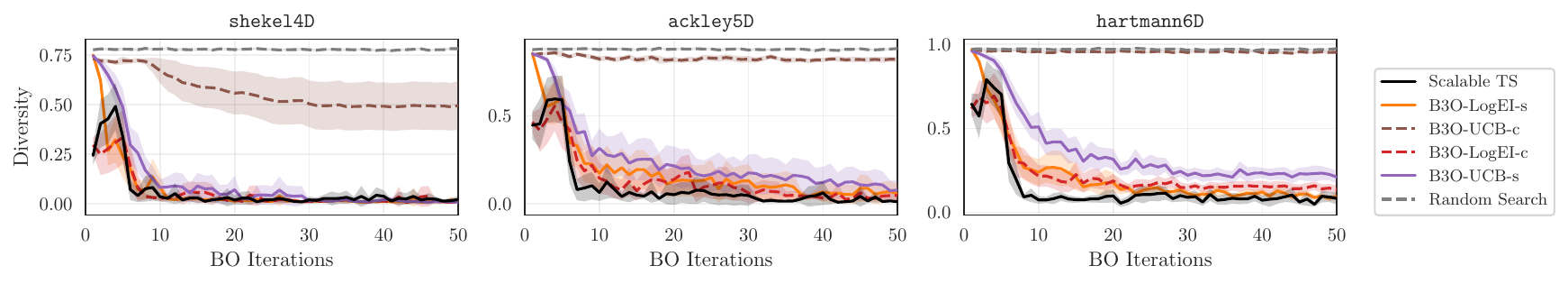}
    \caption{Batch diversity (average pairwise distance) on Shekel (4D, left), Ackley (5D, middle), and Hartmann (6D, right) benchmarks. B3O variants maintain higher diversity throughout the optimization compared to scalable TS, which exhibits rapid diversity loss.}
    \label{fig:diversity_benchmarking}
\end{figure}

\begin{figure}[h]
    \centering
    \includegraphics[width=0.9\linewidth]{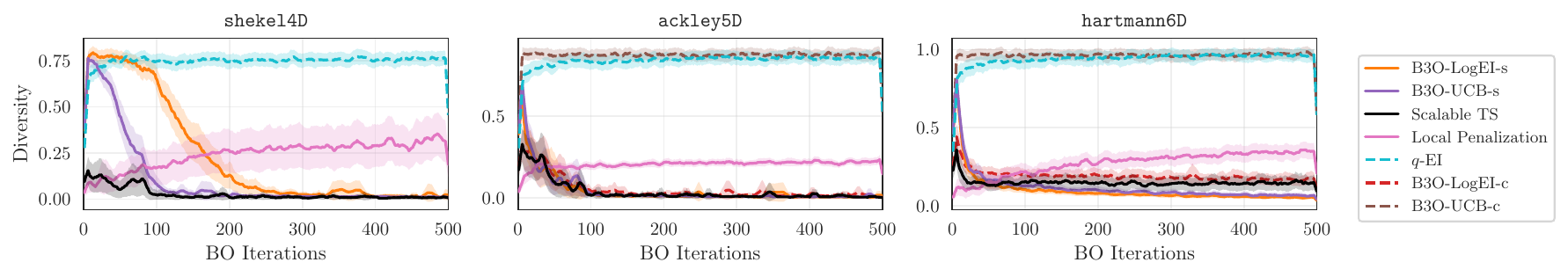}
    \caption{Batch diversity (mean pairwise distance) across BO iterations for B3O variants and batch BO baselines with a small batch size of $B=5$ on Shekel-4D, Ackley-5D and Hartmann-6D benchmarks. Curves are smoothed with a rolling average for clarity.}
    \label{fig:small_batch_diversity}
\end{figure}

% \newpage
% \input{checklist.tex}

\end{document}